\documentclass[12pt]{amsart}

\usepackage[margin=2.8cm]{geometry}

\usepackage{amsmath}
\usepackage{hyperref}
\usepackage{color}
\usepackage{subcaption}

\usepackage{algorithm}
\usepackage{algorithmic}
\newtheorem{theorem}{Theorem}

\newtheorem{lemma}{Lemma}
\newtheorem*{remark}{Remark}

\usepackage{graphicx}
\usepackage{hyperref}

\usepackage{mathptmx}

\newcommand{\cD}{\mathcal{D}}

\newcommand{\R}{\mathbb{R}}
\newcommand{\eps}{\epsilon}

\usepackage{bbm}

\newcommand{\RR}{\mathbb{R}}

\newcommand{\CS}{\mathbb{S}}
\newcommand{\1}{\mathbbm{1}}
 \newcommand{\rv}{\right\Vert}
 \newcommand{\lv}{\left\Vert}
 
 \newcommand{\an}[1]{\textcolor{black}{#1}}
 
  \newcommand{\In}[1]{\textcolor{black}{#1}}

  \newcommand{\tu}{^{(t)}}

  \newcommand{\lp }{\left(}
  
  \newcommand{\rp }{\right)}

\newcommand{\CP}{\mathbb{P}}
\DeclareMathOperator{\sign}{sign}

\newcommand{\s}{\sigma}

\newcommand{\G}{\Gamma}
\DeclareMathOperator{\ReLU}{ReLU}
\DeclareMathOperator{\marg}{marg}
\DeclareMathOperator{\sig}{sigmoid}

\global\long\def\con{\subset}
\global\long\def\a{\rightarrow}

\begin{document}

\title{On Symmetry and Initialization for Neural Networks}


\author{Ido Nachum}
\address{Department of Mathematics, Technion-IIT}
\email{idon@tx.technion.ac.il}

\author{Amir Yehudayoff}
\address{Department of Mathematics, Technion-IIT}
\email{amir.yehudayoff@gmail.com}
\thanks{Supported by ISF grant 1162/15. This work was done while A.Y.\ was visiting the Simons Institute for the Theory of Computing.}


\maketitle

\begin{abstract}
This work provides an additional step in the theoretical
understanding of neural networks. We consider neural networks with one hidden layer and show that when learning symmetric functions, one can choose initial conditions so that standard SGD training efficiently produces generalization guarantees. We empirically verify this and show that this does not hold when the initial conditions are chosen at random. The proof of convergence investigates the interaction between the two layers of the network. Our results highlight the importance of using symmetry in the design of neural networks.
\end{abstract}


\section{Introduction}


Building a theory that can help to understand neural networks 
and guide their construction is one of the current 
challenges of machine learning.
Here we wish to shed some light on the role
symmetry plays in the construction of neural networks.
It is well-known that symmetry can be used
to enhance the performance of neural networks. 
For example, convolutional neural networks (CNNs) (see \cite{lecun2}) 
use the translational symmetry of images
to classify images better than fully connected neural networks.
Our focus is on the role of symmetry in 
the initialization stage.
We show that symmetry-based initialization can be the difference between failure and success.


%
%
%
%
%
%
%
%
%
%

On a high-level, the study of neural networks
can be partitioned to three different aspects.
\begin{description}
\item[{\em Expressiveness}] Given an architecture, what are the functions it can approximate well?
\item[{\em Training}] Given a network with a ``proper'' architecture, can the network fit the training data and in a reasonable time?
\item[{\em Generalization}] Given that the training seemed successful, 
will the true error be small as well?
\end{description}
We study these aspects for
the first ``non trivial'' case of neural networks,
networks with one hidden layer.
We are mostly interested in the initialization phase.
If we take a network with the appropriate architecture, we can always initialize it to the desired function. 
A standard method \In{(that induces a non trivial learning problem)} is using random weights to initialize the network.
A different reasonable choice is to require the initialization
to be useful for an entire class of functions.
We follow the latter option.


Our focus is on the role of symmetry.
We consider the following class of symmetric functions

$$\CS  = \CS_n = \Big\{    \sum_{i=0}^n a_i \cdot \1 _{|x| = i} : a_1,\ldots,a_n \in \{ \pm 1\} \Big\} ,$$
where $x\in \{ 0,1 \}^n$ and $|x| = \sum_i x_i$.
The functions in this class 
are invariant under arbitrary permutations of the input's
coordinates.
The parity function $\pi(x)=(-1)^{|x|}$ 
and the majority function 
are well-known examples of symmetric functions.

%







\emph{Expressiveness} for this class was explored by \cite{Minsky}. They showed that the parity function cannot be represented using a 
network with limited ``connectivity''.
Contrastingly, if we use a fully connected network with one hidden layer and  a common activation function (like $\sign$, $\sig$, or $\ReLU$) only $O(n)$ neurons are needed. We provide such
explicit representations for all functions in $\CS$;
see Lemmas~\ref{sig_lem} and~\ref{re_lem}.



We also provide useful information
on both the {\em training} phase and {\em generalization} capabilities
of the neural network.
We show that, with proper initialization, 
the training process (using standard SGD)
efficiently converges to zero empirical error,
and that consequently the network
has small true error as well.

\begin{theorem}
\label{main} 
{There exists a constant $c>1$ so that the following holds.}
There exists a network with one hidden layer, $cn$ neurons with $\sig$ or $\ReLU$ activations, and an initialization such that for all distributions $\cD$ over $X = \{0,1\}^n$ and all functions $f\in \CS$
with sample size $m \geq c (n+\log (1/\delta))/\epsilon$,
after performing $poly(n)$ SGD updates with a fixed step size $h = 1/poly(n)$ it holds that 
\[ \underset{x^m \sim \cD ^m}{P} \Big( \Big\lbrace  S :  \underset{x\sim \cD}{\Pr}( N_S(x) \neq f(x)) > \epsilon \Big\rbrace  \Big) < \delta \]   
where $S=\{(x_1,f(x_1)),...,(x_m,f(x_m))\}$
and $N_S(x)$ is the network after training over $S$.
\end{theorem}

The number of parameters in the network described in 
Theorem~\ref{main} is $\Omega(n^2)$. So in general one could expect overfitting when the sample size is as small as $O(n)$. 
Nevertheless, the theorem provides generalization guarantees,
even for such a small sample size.

The initialization phase plays an important role in proving Theorem \ref{main}. To emphasize this, we report an empirical phenomenon
(this is ``folklore'').
We show that a network cannot learn parity from a random initialization (see Section \ref{par_hard}). On one hand, if the network size is big, we can bring the empirical error to zero (as suggested in \cite{no_bad}), but the true error is close to $1/2$. On the other hand, if its size is too small, the network is not even able to achieve small empirical error (see Figure \ref{fig_par}). We observe a similar phenomenon also for a random symmetric function. 
An open question remains: why is it true that a sample of size polynomial in $n$ does not suffice to learn parity (with random
initialization)? 


A similar phenomenon was theoretically explained
by \cite{Shamir} and \cite{SongVW}.
The parity function belongs to the class 
of all parities $$\CP =\CP_n=\{\pi_{s}(x)=(-1)^{s \cdot x} : s\in X\}$$
where $\cdot$ is the standard inner product.
This class is efficiently PAC-learnable with $O(n)$ samples using Gaussian elimination. 
A continuous version of $\CP$ 
was studied by \cite{Shamir} and \cite{SongVW}. To study the training phase,
they used a generalized notion of \emph{statistical queries} (SQ);
see \cite{Kearns}. In this framework, 
they show that most functions in the class ${\CP}$ cannot be efficiently learned
(roughly stated, learning the class requires 
an exponential amount of resources). 
This framework, however, does not seem to capture 
actual training of neural networks using SGD. For example, 
it is not clear if one SGD update
corresponds to a single query in this model.
In addition,
typically one receives a dataset and performs the training by going over it many times, whereas
the query model estimates the gradient using a fresh batch of samples in each iteration. 
The query model also assumes the noise 
to be adversarial, an assumption that does not necessarily hold in reality.
Finally, the SQ-based lower bound
holds for every initialization
(in particular, for the initialization we use here),
so it does not capture the efficient training
process Theorem~\ref{main} describes.

Theorem~\ref{main} shows, however, that
with symmetry-based initialization, parity can be efficiently learned.
So, in a nutshell, parity can not be learned as part of ${\CP}$,
but it can be learned as part of $\CS$.
One could wonder why  the hardness proof for $\CP$ cannot be applied for $\CS$ as both classes consist of many input sensitive functions.  The answer lies in the fact that ${\CP}$ has a far bigger statistical dimension than $\CS$ (all functions in $\CP$ are orthogonal to each other, unlike $\CS$). 

The proof of the theorem utilizes the different behavior of the
two layers in the network. 
SGD is performed using a step size $h$ that is polynomially small in $n$.
The analysis shows that in a polynomial number of steps
that is {\em independent} of the choice of $h$
the following two properties hold:
(i) the output neuron reaches a ``good'' state and
(ii) the hidden layer does not change in a ``meaningful'' way.
These two properties hold when $h$ is small enough.
In Section \ref{over}, we experiment with large values
of $h$. We see that, although the training error is zero, the 
true error becomes large.


Here is a high level description of the proof.
The $\ell$ neurons in the hidden layer define
an ``embedding'' of the inputs space $X = \{0,1\}^n$
into $\R^\ell$ (a.k.a.\ the feature map). This embedding changes in time
according to the training examples and process.
The proof shows that if 
at any point in time this embedding 
has good enough margin,
then training with standard SGD quickly converges.
This is explained in more detail in Section~\ref{sec:train}.
It remains an interesting open problem to understand
this phenomenon in greater generality,
using a cleaner and more abstract language.

%

\subsection{Background}

To better understand
the context of our research, 
we survey previous related works.


\In{The expressiveness and limitations of neural networks were studied in several works such as \cite{recht, telg, eldan} and \cite{ABMM}.}
Constructions of small $\ReLU$ networks for the parity function
appeared in several previous works, such as
 \cite{parity_arch}, \cite{parity_arch2}, \cite{parity_arch3}
 and \cite{parity_arch4}. 
Constant depth circuits for the parity function were also studied in the 
context of computational complexity theory,
see for example \cite{Par1}, \cite{Par2} and \cite{Par3}. 

  
%
The training phase of neural networks was also studied in many works. Here we list several works that seem
most related to ours. \cite{daniely} analyzed SGD for general neural network architecture and showed that the training error can be nullified, e.g., for the class of bounded degree polynomials
(see also~\cite{andoni}). \cite{jacot} studied neural tangent kernels (NTK), an infinite width analogue of neural networks. \cite{du} showed that randomly initialized shallow $\ReLU$ networks nullify the training error, as long as the number of samples is smaller than the number of neurons in the hidden layer. Their analysis only deals with optimization over the first layer (so that the weights of the output neuron
are fixed). \cite{chizat} provided another analysis 
of the latter two works. \cite{allen} showed
that over-parametrized neural networks can achieve zero training error, as as long as the data points are not too close to one another and the 
weights of the output neuron are fixed. \cite{zou} provided guarantees for zero training error, assuming the two classes are separated by a positive margin.

Convergence and generalization 
guarantees for neural networks were 
studied in the following works. \cite{sgd_lsd} studied linearly separable data. \cite{sgd_lsd2} studied well separated distributions. \cite{allen_g} gave generalization guarantees in expectation for SGD.  \cite{arora} gave data-dependent generalization bounds for GD.
All these works optimized only over the hidden layer (the output layer is fixed after initialization).

Margins play an important role in learning,
and we also use it in our proof.
\cite{marg_nn_t1}, \cite{marg_nn_t2}, \cite{marg_nn_t3} and \cite{marg_nn_t4} gave generalization bounds for neural networks that are based on their margin when the training ends. From a practical perspective, \cite{max_nn_app1},
\cite{max_nn_app2} and \cite{max_nn_app4} suggested different training algorithms that optimize the margin.

As discussed above, it seems difficult for neural networks to learn parities. \cite{SongVW} and \cite{Shamir} demonstrated this using the language statistical queries
(SQ).
This is a valuable language,
but it misses some central aspects of training neural networks.
SQ seems to be closely related to GD,
but does not seem to capture SGD.
SQ also shows that many of the parities functions
$\otimes_{i \in S} x_i$ are difficult to learn,
but it does not imply that {\em the} parity function
$\otimes_{i \in [n]} x_i$ is difficult to learn.
 \cite{deep_limit} demonstrated 
a similar phenomenon in a setting that is closer to the ``real life'' mechanics of neural networks.

We suggest that taking the symmetries of the learning problem into account can make the difference between failure and success. 
Several works 
suggested different neural architectures that take symmetries into account;
see \cite{deepsets}, \cite{deep_sym}, and \cite{deep_equi}.

\section{Representations}

Here we describe efficient representations for symmetric
functions by network with one hidden layer. These representations 
are also useful later on, when we study the training process.
We study two different activation functions,
$\sig$ and $\ReLU$
(similar statement can be proved for other activations,
like $\arctan$). Each activation function 
requires its own representation,
as in the two lemmas below.

 \begin{figure}[ht]
\vskip 0.1in
\begin{center} 
\centerline{\includegraphics[ scale=0.4]{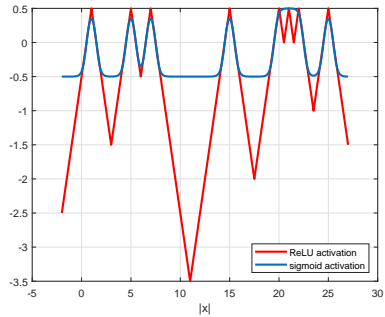}}
\caption{Approximations of the symmetric function $f_A
=\sign(\sum_{i\in A}\1 _{|x|=i}- 0.5)$ by $\sig$ and $\ReLU$ activations 
for $A=\{1, 5, 7, 15, 20, 21, 22, 25  \}$.
}
\label{app_exm}
\end{center}
\vskip -0.1in
\end{figure}

\subsection{Sigmoid}
%

We start with the activation $\sigma(\xi)=\frac{1}{1+\exp(-\xi)}$,
since it helps to understand the construction
for the $\ReLU$ activation.
The building blocks of the symmetric functions
are indicators of $|x|=i$ for $i \in \{0,1,\ldots,n\}$.
An indicator function is essentially a sum of two $\sig$ functions: 
$$\sign(\1_{|x| = i}-0.5)=\sign(\Delta_i-0.5),$$
where
$\Delta_i(x)= \s (5(|x| - i+0.5)) + \s (5(i+0.5 - |x|))-1$.

\begin{lemma}\label{sig_lem}
The symmetric function $f_A$ satisfies $f_A(x)=\sign(-0.5+ \sum_{i\in A}  \Delta_i(x))$.
\end{lemma}

A network with one hidden layer of $2n+3$ neurons with $\sig$ activations is sufficient to represent any symmetric function.

\begin{proof}
For all $k\in A$ and $x\in X$ of weight $k$, 
 \[\sum_{i\in A} \Delta_i(k)\geq \Delta_k(x)  = 2\s (5\cdot 0.5)-1 > 0.84 ; \] 
the first inequality holds since $\Delta_i(x)\geq 0$ for all $i$ and $x$. 
For all $k \not \in A$ and $x\in X$ of weight $k$, 
\begin{align*}
\sum_{i\in A} \Delta_i(x) 
& =  
\sum_{k < i\in A} \s (5\cdot( k -i + 0.5)) 
 +  \sum_{k > i\in A} \s (5\cdot(  i + 0.5  - k)) \\ 
 & +  \sum_{k < i\in A} \left[ \s (5\cdot(  i + 0.5  - k))-1\right]   + \sum_{k > i\in A} \left[ \s (5\cdot( k -i + 0.5))-1 \right]   \\ 
&  <   \sum_{k < i\in A} \s (5\cdot( k -i + 0.5)) +  \sum_{k > i\in A} \s (5\cdot(  i + 0.5  - k)) \\  
& <  
 2 \sum_{i=1}^{\infty} \exp( 5\cdot(  -i + 0.5) )\\ 
 & =    2\exp(-2.5)/(1-\exp(-5))  <0.17;
\end{align*}
the first equality follows from the definition, the first inequality neglects  the negative sums, and the second inequality follows 
because $\exp(\xi) > \s(\xi) $ for all $\xi$.

%

\end{proof}


\subsection{ReLU}

An indicator function  can be represented using $ \ReLU(\xi)=\max \{ 0,\xi \}$
as
$
\sign( \G _i+0.5)$,
where 
$$\Gamma_i(x)= -\ReLU (|x|- i) - \ReLU (i - |x|) .$$
A natural idea is to take a linear combination (similarly to the $\sig$) to get 
general functions in $\CS$. 
However, this fails because the $\ReLU$ function is unbounded.  
The following lemma states the needed correction.

\begin{lemma}\label{re_lem}
Let $A = \{i_1<i_2<...<i_t\} \subseteq [n]$ for $t >1$.
Define
$B=\{ (i_1+i_2)/2,...,(i_{t-1}+i_t)/2 \}$.
The symmetric function 
$$f_A= \sign(- 0.5 + \sum_{i\in A}\1 _{|x|=i})$$
can be represented as
$$f_A = \sign(  (i_t-i_1)/2 +0.5+\sum_{i\in A} \G _i - \sum_{i\in B} \G _i).$$
\end{lemma}

The lemma shows that a network with one hidden layer of $4n+3$ $\ReLU$ neurons is sufficient to represent any function
in $\CS$.
The coefficient of the $\ReLU$ gates are
$\pm 1$ in this representation.

\begin{proof} 
The proof proceeds in two parts. 
The first part shows the function $ (i_t-i_1)/2 +0.5+\sum_{i\in A} \G_i - \sum_{i\in B} \G_i$ is constant for all $x\in X$ so that $|x| \in A$. The second part shows that this function equals $0.5$ for all $x$ so that $|x| \in A$ and that it is negative for all $x\in X$ that satisfy $|x| \notin A$.

For the first part, denote by $F_A(j)$
the value of the symmetric function $\sum_{i\in A} \G_i - \sum_{i\in B} \G_i$ on inputs of weight $j$. 
By induction, assume that
$F_A(i_1)=...=F_A(i_m)$ for some $1 \leq m <t$.  
Think of $F_A$ as a univariate function of
the real variable $\xi$.
This function is differentiable 
for all $i_m< \xi <(i_m+i_{m+1})/2$: 
\begin{align*}
F_A'(\xi) & = -\left[ R'(\xi-i_m)-R'(i_m-\xi) \right]\\
 & - \sum_{i_{m}>i\in A} \left[ R'(\xi-i)-R'(i-\xi) \right] 
 -\sum_{i_{m}<i\in A} \left[ R'(\xi-i)-R'(i-\xi) \right] \\
&  +  \sum_{i_{m}>i\in B} \left[ R'(\xi-i)-R'(i-\xi) \right] 
  + \sum_{i_{m}<i\in B} \left[ R'(\xi-i)-R'(i-\xi) \right] \\
& =   
- R'(\xi-i_m) -\sum_{i_{m}>i\in A}  R'(\xi-i) +\sum_{i_{m}<i\in A}  R'(i-\xi)\\
&   +    \sum_{i_{m}>i\in B} R'(\xi-i) - \sum_{i_{m}<i\in B} R'(i-\xi) \\
& =      -1 ;
\end{align*} 
the first equality follows from the definition of $\G_i$, the second equality follows from the definition of the $\ReLU$ function, and the last equality holds since the first and third sum cancel each other and  the second and fourth sum as well. 
In a similar manner, 
for all $(i_m+i_{m+1})/2<\xi<i_{m+1}$,
we have $F_A'(\xi) = 1$.
So, integrating over $\xi$ concludes the induction $F_A(i_{m+1})=F_A(i_{m})+\int_{i_m}^{i_{m+1}} F_A'(\xi)d\xi=F_A(i_{m})$.

For the second part, 
we start by proving that $F_{A}(i_1)= -(i_t-i_1)/2$.
Let $A_m=\{i_1,...,i_m\}$. 
By definition, $\G_{i_1}(i_1)=0$. 
For $A_2$, we have
$$F_{A_2}(i_1)=\G_{i_2}({i_1})-\G_{(i_1+i_2)/2}({i_1})=-(i_2-i_1)+[(i_2+i_1)/2-i_1]=-(i_2-i_1)/2.$$ 
Induction on $m$ can be used to prove
that $F_{A}(i_1)= -(i_t-i_1)/2$.
Now, by the derivatives calculated in the first part, for $k \notin A$ it holds that $F_A(k) \leq F_A(i_1)-1$.

\end{proof}

\section{Training and Generalization}\label{sec:train}

The goal of this section is to describe
a small network with one hidden layer
that (when initialized property)
efficiently learns symmetric functions
using a small number of examples
(the training is done via SGD).

\subsection{Specifications}

Here we specify the architecture, initialization
and loss function that is implicit in our main result
(Theorem~\ref{main}).

To guarantee convergence of SGD, we need to start with ``good'' initial conditions. The initialization we pick depends
on the activation function it uses, and is chosen with resemblance to Lemma
~\ref{re_lem} for $\ReLU$.   
On a high level, this indicates that understanding the class
of functions we wish to study in term of ``representation''
can be helpful when choosing the architecture of a neural network
in a learning context.

The network we consider has one hidden layer.
We denote by $w_{ij}$ the weight between coordinate $j$ of the input and neuron $i$ in the hidden layer. We denote $W$ this matrix of weights.
We denote by $b_i$ the bias of neuron $i$ of the hidden layer. We denote $B$ this vector of weights.
We denote by $m_{i}$ is the weight from neuron $i$ in the hidden layer to the output neuron. We denote $M$ this vector of weights.
We denote by $b$ the bias of the output neuron.

Initialize the network as follows:
 The dimensions of  $W$ are $(4n+2) \times n$.  For all $1\leq i \leq (4n+2)$ and $1\leq j \leq n $,
we set  
$$w_{ij}= (-1)^{i+1} \quad
\text{and} \quad b_i=0.5(-1)^i  \lfloor (i-1)/2 \rfloor .$$
We set $M=0$ and $b=0$.

To run SGD, we need to choose a loss function.
We use the \emph{hinge loss}, $$L(x,f)=\max\{0, -f(x)(v_x\cdot M+b)+\beta\},$$ where $v_x = \ReLU(W x+B)$ is the output of the hidden layer on input $x$ and $\beta>0$ is a parameter of confidence.


 \subsection{Margins}\label{margin}
 
A key property in the analysis is the `margin' of the hidden layer with respect to the function being learned.

A map $Y: V \a \{\pm 1\}$ over a finite set $V \con \mathbb{R} ^d $ is linearly\footnote{A standard ``lifting'' that adds a coordinate with $1$
to every vector allows to translate the affine case to the linear case.}
separable if there exists {$w\in\RR^d$} 
such that $\sign(w \cdot v)=Y(v)$ for all $v\in V$.
When the Euclidean norm of $w$ is $\|w\|=1$,
the number $\marg (w,Y) = {  \min_{v\in V} {Y(v) w\cdot v }}$ is the margin 
of $w$ with respect to $Y$.
The number $\marg(Y) = \sup_{w \in \RR^d : \|w\|=1} \marg(w,Y)$
is the margin of $Y$. 

We are interested in the following set $V$ in $\RR^d$.
Recall that $W$ is the weight matrix between the input layer and the hidden layer, and that $B$ is the relevant bias vector.
Given $W,B$,
we are interested in the set $V=\{ v_x : x\in X\}$,
where $v_x = \ReLU(W x+B)$.
In words, we think of the neurons in the hidden layer
as defining an ``embeding'' of $X$ in Euclidean space.
A similar construction works for other activation functions.
We say that $Y : V \to \{\pm 1\}$ agrees with $f\in \CS$ if for all $x\in X$ it holds that $Y(v_x)=f(x)$.

The following lemma bounds from below the margin of the initial $V$.

\begin{lemma}
\label{lem:part}
If $Y$ is a partition that agrees with some function in $\CS$
for the initialization described above 
then $\marg(Y) \geq \Omega (1/n)$.
\end{lemma}

\begin{proof}
By Lemmas \ref{sig_lem} and \ref{re_lem}, we see that any function in $\CS$ can be represented with a vector of weights $M\in [-1,1]^{\Theta(n)}$ of the output neuron together with a bias $b \in [-(n+1),n+1]$. 
These $M,b$ induce a partition $Y$ of $V$.
Namely, $Y(v_x) M\cdot v_x + b >0.25$ for all $x \in X$.
Since $\Vert(M,b)\Vert=O(n)$ we have our desired result.  
\end{proof}

\subsection{Freezing the Hidden Layer}

Before analyzing the full behavior of SGD, 
we make an observation:
if the weights of the hidden layer are fixed
with the initialization described above, 
then Theorem~\ref{main} holds for SGD with batch size $1$. 
{This observation, unfortunately, does
not suffice to prove Theorem~\ref{main}.
In the setting we consider, the training of the neural network
uses SGD without fixing any weights.
This more general case is handled in the next section.}
The rest of this subsection is devoted for
explaining this observation.

\cite{Nov} showed that that the perceptron algorithm \cite{Rosenblatt}
makes a small number of mistakes
for linearly separable data with large margin.
For a comprehensive survey of the perceptron algorithm and its variants, see~\cite{perc_comp}.

Running SGD with the hinge loss induces the same update rule
as in a modified  
perceptron algorithm, Algorithm~\ref{alg:Per}.



\begin{algorithm}
\caption{The modified perceptron algorithm}
 \label{alg:Per}
%
%

\begin{algorithmic}
 \STATE\textbf{Initialize:} $w^{(0)}=\vec{0}$, $t=0$, $\beta>0$ and $h>0$
 \WHILE{$\exists v \in V$ with $Y(v) w^{(t)} \cdot v \leq \beta $} 
 \STATE{$w^{(t+1)}=w^{(t)}+Y(v) v h$} 
 \STATE{$t = t+1$}
 \ENDWHILE 
 
 
 \textbf{return} $w^{(t)}$

 \end{algorithmic}

\end{algorithm}


Novikoff's proof can be generalized to any $\beta>0$ and batches of any size to yield the following theorem; see 
\cite{perc_mod1,perc_mod2} and appendix \ref{proof_perc}.

\begin{theorem}\label{R-perc}
For $Y : V \to \{\pm 1\}$ with margin $\gamma>0$ 
and step size $h>0$, the modified perceptron algorithm performs at most $\frac{2\beta h +(Rh)^2}{(\gamma h)^2}$ updates and achieves a margin of at least $\frac{\gamma  \beta h}{2\beta h+(Rh)^2} $,
where $R = \max_{v \in V} \|v\|$.
\end{theorem}

So, when the weights of the hidden layer are fixed,
Lemma~\ref{lem:part} implies that the number
of SGD steps is at most polynomial in $n$.

 

\subsection{Stability}\label{proof}


When we run  SGD on the entire network,
the layers interact. For a $\ReLU$ network at time $t$, the update rule for $W$ is as follows. 
If the network classifies the input $x$
correctly with confidence more than $\beta$,
no change is made.
Otherwise, we change the weights in $M$ by $\Delta M= yv_xh$, where $y$ is the true label and $h$ is the step size. If also neuron $i$ of the hidden \emph{fired} on $x$, we update its incoming weights by $\Delta W_{i,:} = ym_ixh$. 
These update rules define the following dynamical system: 

\begin{align}  \label{dyn1}
W^{(t+1)} &=W^{(t)}+y \left( \lp M \tu \rp ^T \circ H\lp W\tu x+B\tu  \rp \right) x^T h  \\
B^{(t+1)}&= B^{(t)}+y\left( \lp M \tu \rp ^T \circ H\lp W\tu x+B\tu  \rp \right) h \\
M^{(t+1)}&=M^{(t)}+y\ReLU\lp W\tu x+B\tu\rp h \label{dyn3} \\
b^{(t+1)}&=b^{(t)}+yh  ,\label{dyn4}
\end{align}
 where $H$ is the Heaviside step function and $\circ$ is the Hadamard pointwise product. {
 
 A key observation in the proof is that the weights of the last layer (\eqref{dyn3} and \eqref{dyn4}) are updated exactly as the modified perceptron algorithm. 
Another key statement in the proof is that if the network has reached a good representation of the input (i.e., the {hidden layer} has a large margin), then the interaction between the layers during the continued training does not impair this representation.  
This is summarized in the following lemma
(we are not aware of a similar statement in the literature).   
}

\begin{lemma}\label{embed}
Let $M=0$, $b=0$, and 
$V=\{ \ReLU(Wx+B) : x\in X  \}$ be a linearly separable embedding of $X$ and with margin $\gamma >0$ by the hidden layer of a neural network of depth two 
with $\ReLU$ activation and weights given by $W,B$.
Let  $R_X=\max_{x\in X} \Vert x \Vert $, 
let $R=\max_{v \in V} \Vert v  \Vert $, and 
$0 < h \leq \frac{\gamma^{5/2} }{100 R^2R_X}$ be the integration step.
Assuming $R_X>1$ and $\gamma \leq 1$, and using $\beta=R^2h$ in the loss function, after $t$ SGD iterations the following hold:
\begin{itemize}
\item[--] Each $v\in V$ moves a distance of at most 
$O( R_X^2 h^2 Rt^{3/2} )$.
\item[--]The norm $\|M^{(t)}\|$
is at most $O(Rh\sqrt{t})$.
\item[--]
The training ends in at most $O(R^2/\gamma^2)$ SGD updates.
\end{itemize}

\end{lemma}

Intuitively, this type of lemma can be useful
in many other contexts. The high level idea
is to identify a ``good geometric structure''
that the network reaches and enables efficient learning. 

\begin{proof}
We are interested in the maximal distance 
the embedding of an element $x\in X$ has moved from its initial embedding:
\begin{align} 
\lv v_x^{(t)}-v_x^{(0)}\rv&=\left\Vert \ReLU(W^{(t)}x+B^{(t)})-\ReLU(W^{(0)}x+B^{(0)}) \right\Vert  \\
& \leq \lv  W \tu -W^{(0)} \rv R_X + \lv  B \tu -B^{(0)} \rv \\
& \leq  \sum_{k=1}^t \left[ R_X\lv W^{(k)} -W^{(k-1)} \rv+\lv B^{(k)} -B^{(k-1)} \rv \right] .\label{dist}   
\end{align}
To simplify equations \eqref{dyn1}-\eqref{dyn4} discussed above, we assume that during the optimization process the  norm of the weights $W$ and $B$ grow at a maximal rate:
\begin{align} \label{dyn1_s}
\lv W^{(t+1)} - W^{(t)} \rv  &=\lv y \left( \lp M \tu \rp ^T \circ H\lp W\tu x+B\tu \rp \right) x^T h  \rv  
\leq  \lv M \tu \rv R_X h ,   \\
\lv B^{(t+1)}- B^{(t)}  \rv &=\lv y  \lp M \tu \rp ^T \circ H\lp W\tu x+B\tu \rp  h \rv \leq    \lv M \tu \rv h \label{dyn2_s} ;
\end{align}
\an{here the norm of a matrix is the $\ell_2$-norm.}

To bound these quantities, we follow the  modified perceptron proof and add another quantity to bound. That is, the maximal norm $R\tu$ of the embedded space $X$ at time $t$ satisfies (by assumption $R_X>1$)
$$R^{(t+1)}\leq R\tu +(1+ R_X^2 )\lv M \tu \rv h \leq  R\tu +2 R_X^2 \lv M \tu \rv h ;$$
\an{we used that the spectral norm of a matrix is at most
its $\ell_2$-norm.}

We assume a worst-case where $R\tu$ grows monotonically at a maximal rate.  By the modified perceptron algorithm and choice $\beta=R^2h$,  
$$\lv M^{(t)} \rv\leq \sqrt{t((R\tu h)^2+2\beta h )}\leq \sqrt{3} R\tu h\sqrt{t}.$$ 
By choice of $h\leq  \frac{\gamma^{5/2} }{100R^2R_X} $  and assuming $t \leq 20R^2/\gamma^2$,
$$R^{(t+1)} \leq R\tu +2\sqrt{3}  R_X^2 R\tu  h^2 \sqrt{t}  \leq R\tu +\frac{2\sqrt{60}}{100^2} R\tu \gamma^4/R^3 .$$
Solving the above recursive equation,
it holds for all $t\leq 20R^2/\gamma^2$,
$$R\tu\leq \lp1+\frac{2\sqrt{60}}{100^2}\gamma^4/R^3\rp^tR     \leq \exp\lp \frac{40\sqrt{60}}{100^2}\gamma^2/R \rp R
\leq 2R.$$
Now, summing equation \ref{dist}, we have 
$$\lv v_x^{(t)}-v_x^{(0)}\rv \leq 2\sqrt{6} R_X^2 h^2 Rt^{3/2} ,$$ since $\sum_{k=1}^t \sqrt k 
\leq t^{3/2}$.  

So in $20R^2/\gamma^2$ updates, the elements embedded by the network travelled at most $\frac{2\cdot20^{3/2}\sqrt{6}}{100^2}\gamma^2\leq 0.05\gamma^2$. 
Hence, the samples the network received kept a margin of $0.9\gamma$ during training (by the assumption $\gamma\leq 1$). 
By choice of the loss function,
SGD changes the output neuron as in the modified 
perceptron algorithm.
By Theorem~\ref{R-perc},
the number of updates is at most $\frac{2R^2+(2R)^2}{0.9\gamma^2}< 20R^2/\gamma^2 $. So, the assumption on $t$ we made during the proof holds.


%
%

%
%
%

\end{proof}

\subsection{Main Result}

\begin{proof}[Proof of Theorem \ref{main}]
There is an unknown distribution $\cD$ over the space $X$. We pick i.i.d.\ examples $S=((x_1,y_1),...,(x_m,y_m))$ where $m \geq c \big( \tfrac{n+\log (1/\delta)}{\epsilon} \big)$ according to $\cD$,
where $y_i = f(x_i)$ for some $f \in \CS$.
Run SGD for $O(n^5)$ steps,
where the step size is $h =O (1/n^6)$
and the parameter of the loss function is $\beta = R^2h$
with $R=n^{3/2}$.

We claim that it suffices to show that at the end of the training 
(i) the network correctly classifies
all the sample points $x_1,\ldots,x_m$, and 
(ii) for every $x \in X$ such that there exists $1\leq i\leq m$ with $|x|=|x_i|$, the network outputs $y_i$ on $x$ as well.
Here is why.
The initialization of the network embeds the space $X$ into $4n+3$  dimensional space (including the bias neuron of the hidden layer). 
Let $V^{(0)}$ be the initial embedding 
$V^{(0)}=\{\ReLU(W^{(0)} x+B^{(0)}): x \in X\}$.
Although $|X|=2^n$, 
the size of $V^{(0)}$ is $n+1$.
The VC dimension of all the boolean functions over 
$V^{(0)}$ is $n+1$. 
Now, $m$ samples suffice
to yield $\eps$ true error for an ERM 
when the VC dimension is $n+1$;
see e.g.~Theorem~6.7 in~\cite{understanding}.
It remains to prove (i) and (ii) above.

By Lemma~\ref{lem:part}, at the beginning of the training, the partition
of $V^{(0)}$
defined by the target $f \in \CS$ has a margin of $\gamma=\Omega(1/n)$.
We are interested in the eventual 
$V^* = \{\ReLU(W^{*} x+B^{*}): x \in X\}$ embedding of $X$ as well.
The modified perceptron algorithm guarantees that after 
$K\leq(2\beta h +(Rh)^2)/(\gamma h)^2 = O(n^5)$ updates, ($M^{*},b^{*}$) separates the embedded sample
$V^{*}_S = \{\ReLU(W^{*}x_i + B^{*}) : 1 \leq i \leq m\}$ with a margin of at least $\gamma/3$. This happens as long as the updates we perform come from a set with maximal norm $R$ and 
with margin at least $\gamma$.
This is guaranteed by Lemma \ref{embed} and concludes the proof of (i).

It remains to prove (ii).
Lemma~\ref{embed} states that as long as less than $K=O(n^5)$ updates were made, the elements in $V$ moved at most $O(1/n^2)$. 
At the end of the training, the embedded sample $V_S$ is separated with a margin of at least $\gamma/3$ 
with respect to the hyperplane defined by $M^{*}$ and $B^{*}$.
Each $v^*_x$ for $x \in X$
moved at most $O(1/n^2) < \gamma/4$.
This means that if $|x| = |x_i|$
then the network has the same output on $x$ and $x_i$.
Since the network has zero empirical error,
the output on this $x$ is $y_i$ as well.

A similar proof is available with $\sig$ activation
(with better convergence rate and larger allowed step size).

\end{proof} 

\begin{remark}
The generalization part of the above proof
can be viewed as a consequence of
sample compression~(\cite{little}). 
Although the eventual network 
 depends on {\em all} examples,
the proof shows that its functionality 
depends on at most $n+1$ examples.
Indeed, after the training, all examples with equal hamming weight
have the same label. 
\end{remark}

\begin{remark}
The parameter $\beta=R^2h$ we chose in the proof may seem odd and negligible. It is a construct in the proof that allows us to bound efficiently the distance that the elements  in $V$ have moved during the training. For all practical purposes $\beta=0$ works as well (see Figure \ref{beta_0}).
\end{remark}

%

\section{Experiments}\label{exp}
We accompany the theoretical results with some experiments.
We used a network with one hidden layer of $4n+3$ neurons, $\ReLU$ activation, and the hinge loss with $\beta=n^3 h$.  In all the experiments, we used SGD with mini-batch of size one
and before each epoch we randomized the sample. The graphs present the training error and the true error\footnote{We deal with high dimensional spaces, so the true error was not calculated exactly but approximated on an independent batch of samples of size $10^4$. }
versus the epoch of the training process. 
In all the comparisons below, we chose a random symmetric function  and a random sample from $X$.

%
%
%
%
%
%
%
%
%
%
%
%

{
\begin{figure}[h]
 
\begin{subfigure}{0.5\textwidth}
\includegraphics[width=0.9\linewidth]{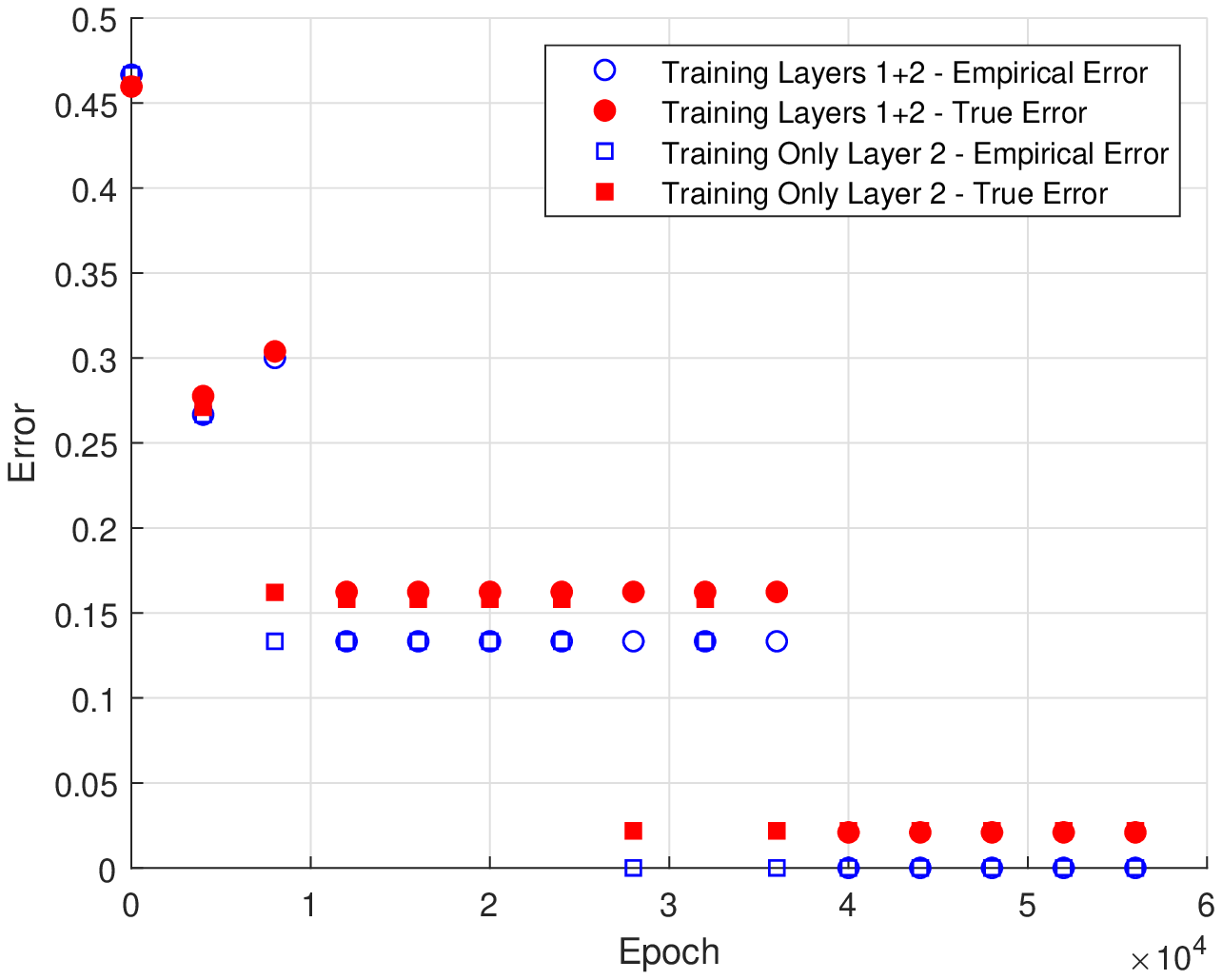} 
\caption{$n=30$ input dimension}
\label{fig:subim1}
\end{subfigure}\hfill
\begin{subfigure}{0.5\textwidth}
\includegraphics[width=0.9\linewidth]{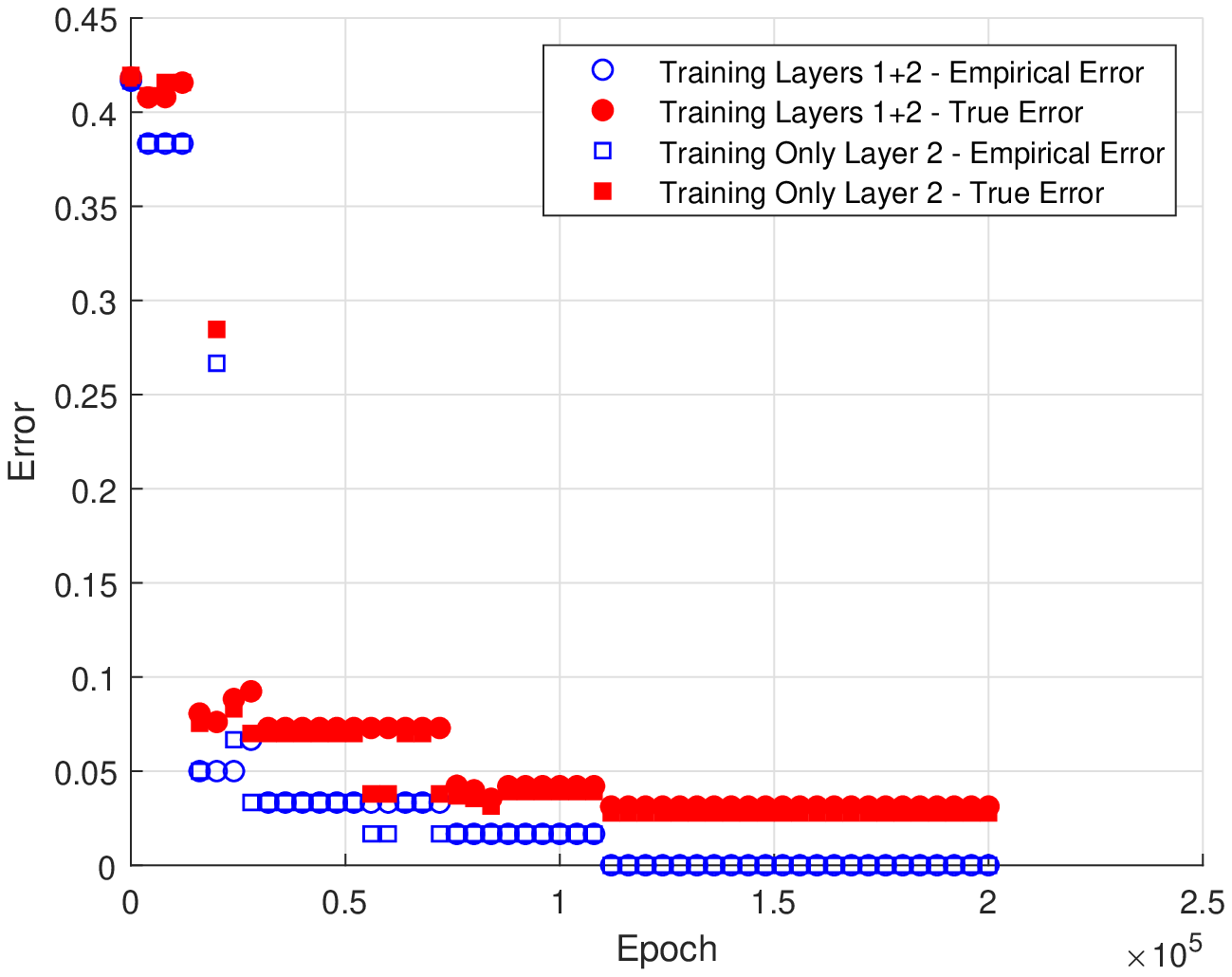}
\caption{$n=60$ input dimension}
\label{fig:subim2}
\end{subfigure}
 
\caption{Error during training for a sample of size $n$.}
\label{fig1}
\end{figure}

}

\begin{figure}[h]
 
\begin{subfigure}{0.32\textwidth}
\includegraphics[width=0.9\linewidth]{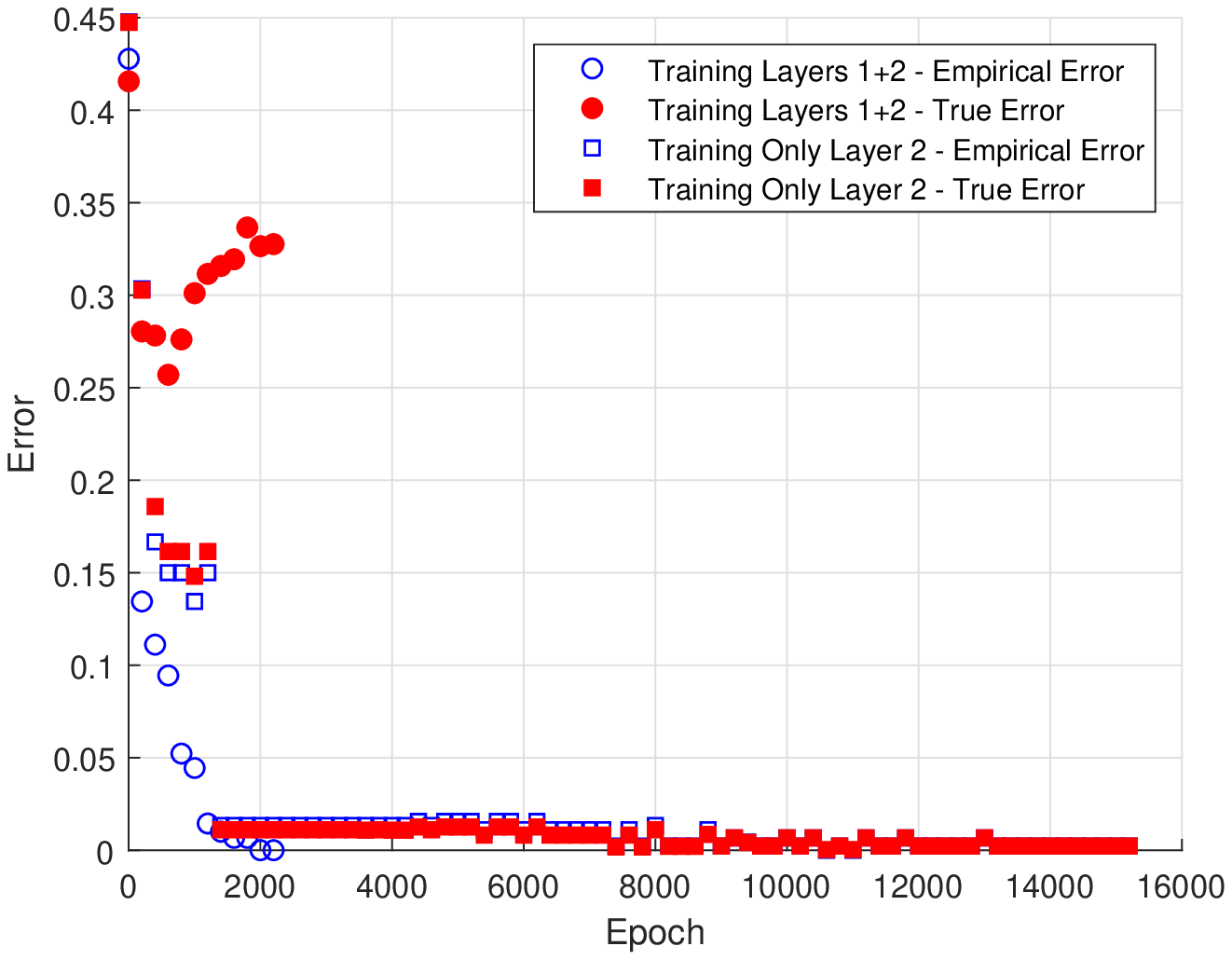} 
\caption{step size $n^{-2}$}

\end{subfigure}
\begin{subfigure}{0.32\textwidth}
\includegraphics[width=0.9\linewidth]{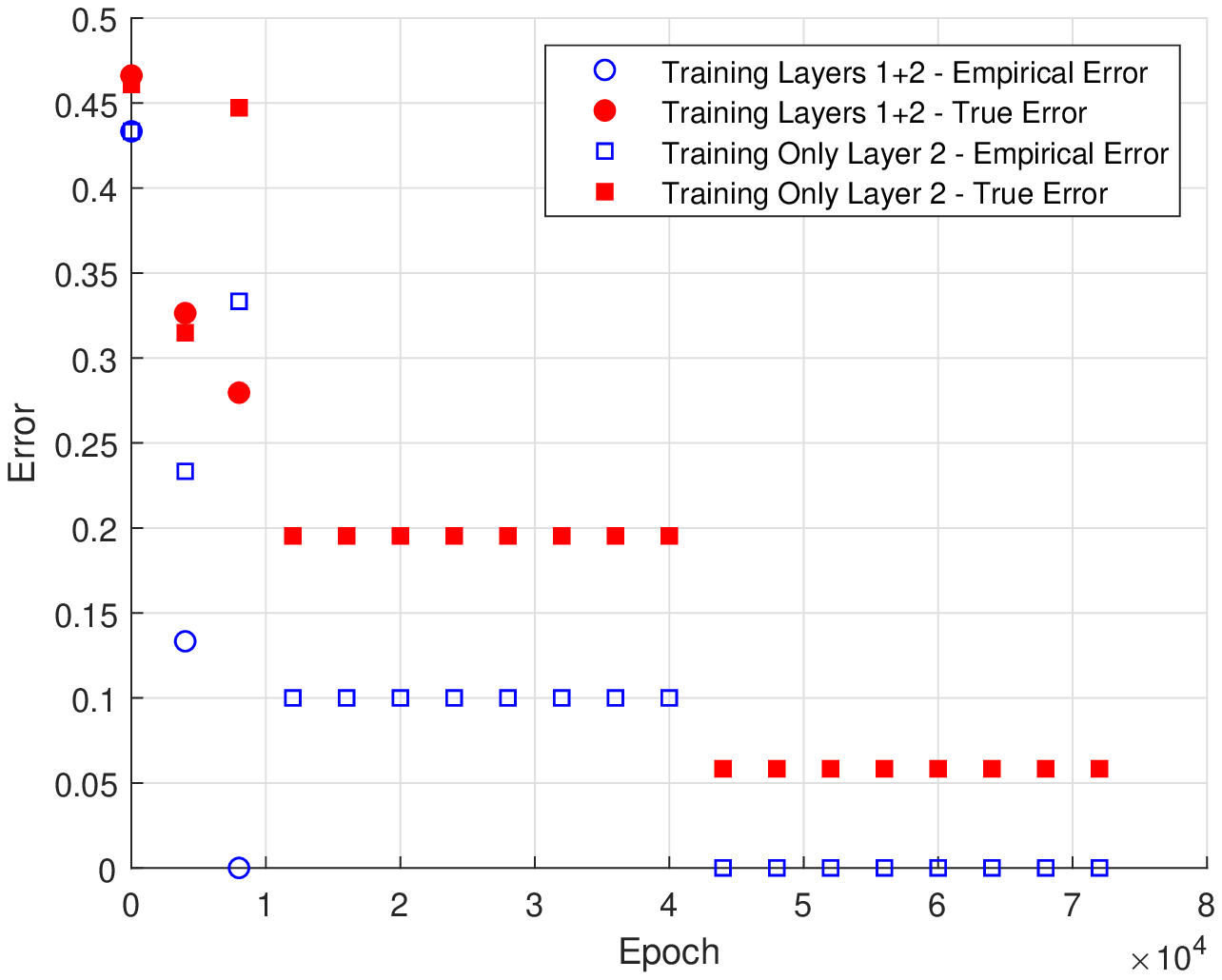}
\caption{step size $n^{-3}$}

\end{subfigure}
\begin{subfigure}{0.32\textwidth}
\includegraphics[width=0.9\linewidth]{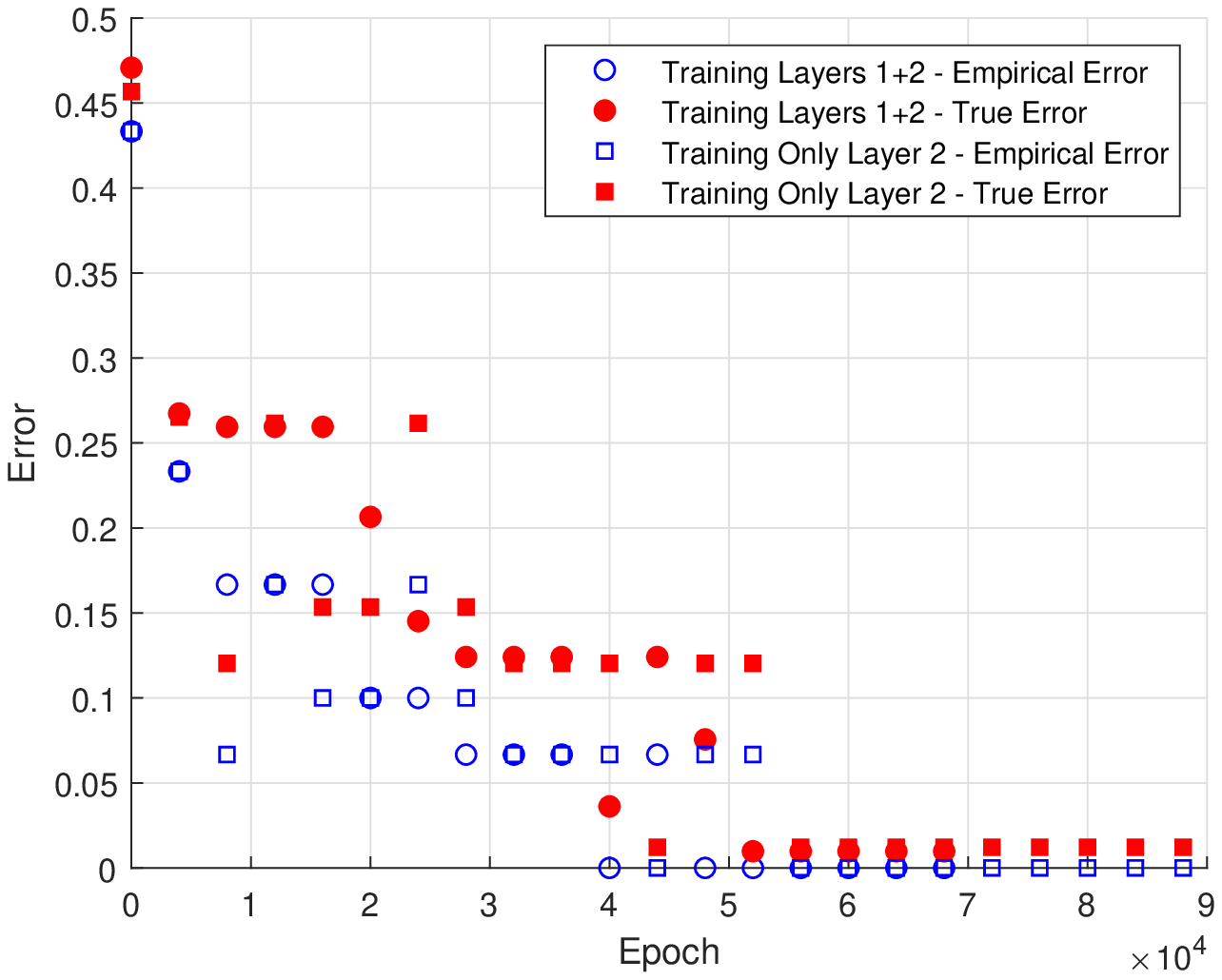}
\caption{step size $n^{-4}$}

\end{subfigure}\label{fig2}

\caption{Error during training for an input dimension and a sample of size $n=30$. }
\label{fig2}
\end{figure}

\begin{figure}[h]
\vskip 0.1in
\begin{center} 
\centerline{\includegraphics[ scale=0.40]{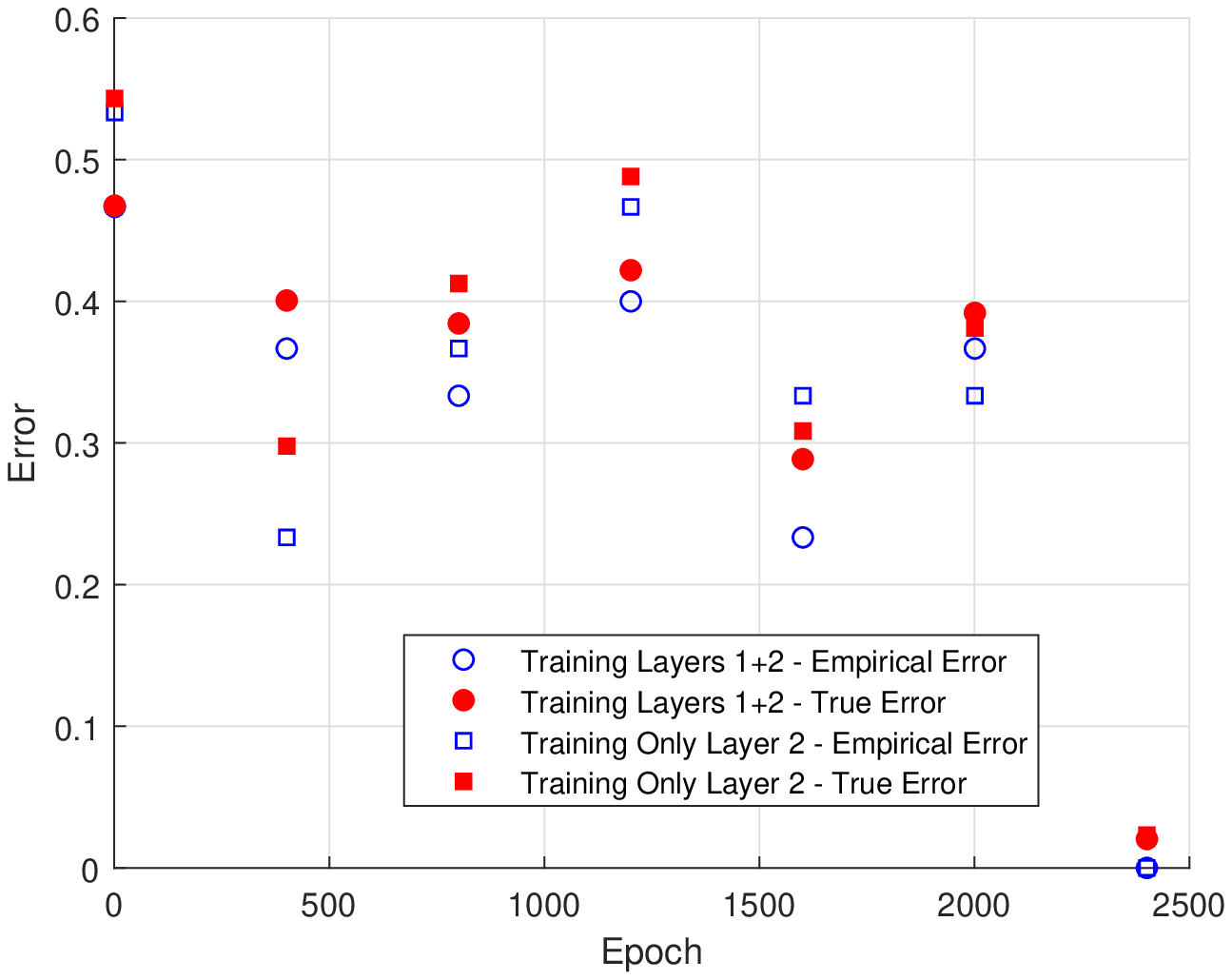}}
\caption{$\beta=0$: error during training for an input dimension and a sample of size $n=30$.
}
\label{beta_0}
\end{center}
\vskip -0.1in
\end{figure}

\begin{figure}[h]

\begin{subfigure}{0.45\textwidth}
\includegraphics[width=0.9\linewidth]{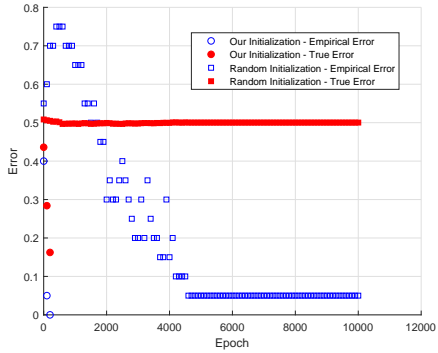} 
\caption{sample of size $n$}
\label{fig:subim1}
\end{subfigure}\hfill
\begin{subfigure}{0.45\textwidth}
\includegraphics[width=0.9\linewidth]{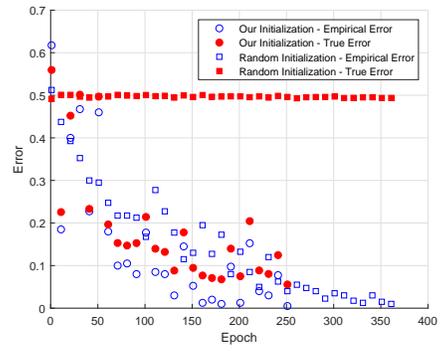}
\caption{sample of size $n^2$}
\label{fig:subim2}
\end{subfigure}

\begin{subfigure}{0.45\textwidth}
\includegraphics[width=0.9\linewidth]{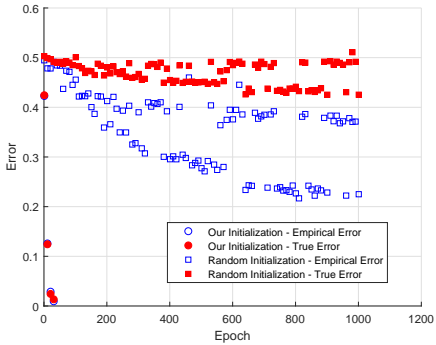} 
\caption{sample of size $n^3$}
\label{fig:subim1}
\end{subfigure}
\hfill
\begin{subfigure}{0.45\textwidth}
\includegraphics[width=0.9\linewidth]{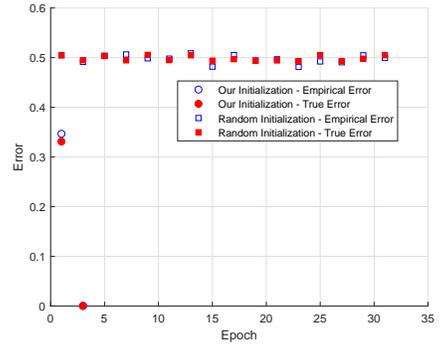}
\caption{sample of size $n^4$}
\label{fig:subim2}
\end{subfigure}
 
\caption{Parity: error during training for input dimension $n=20$. }
\label{fig_par}
\end{figure}

\begin{figure}[h]
 
\begin{subfigure}{0.45\textwidth}
\includegraphics[width=0.9\linewidth]{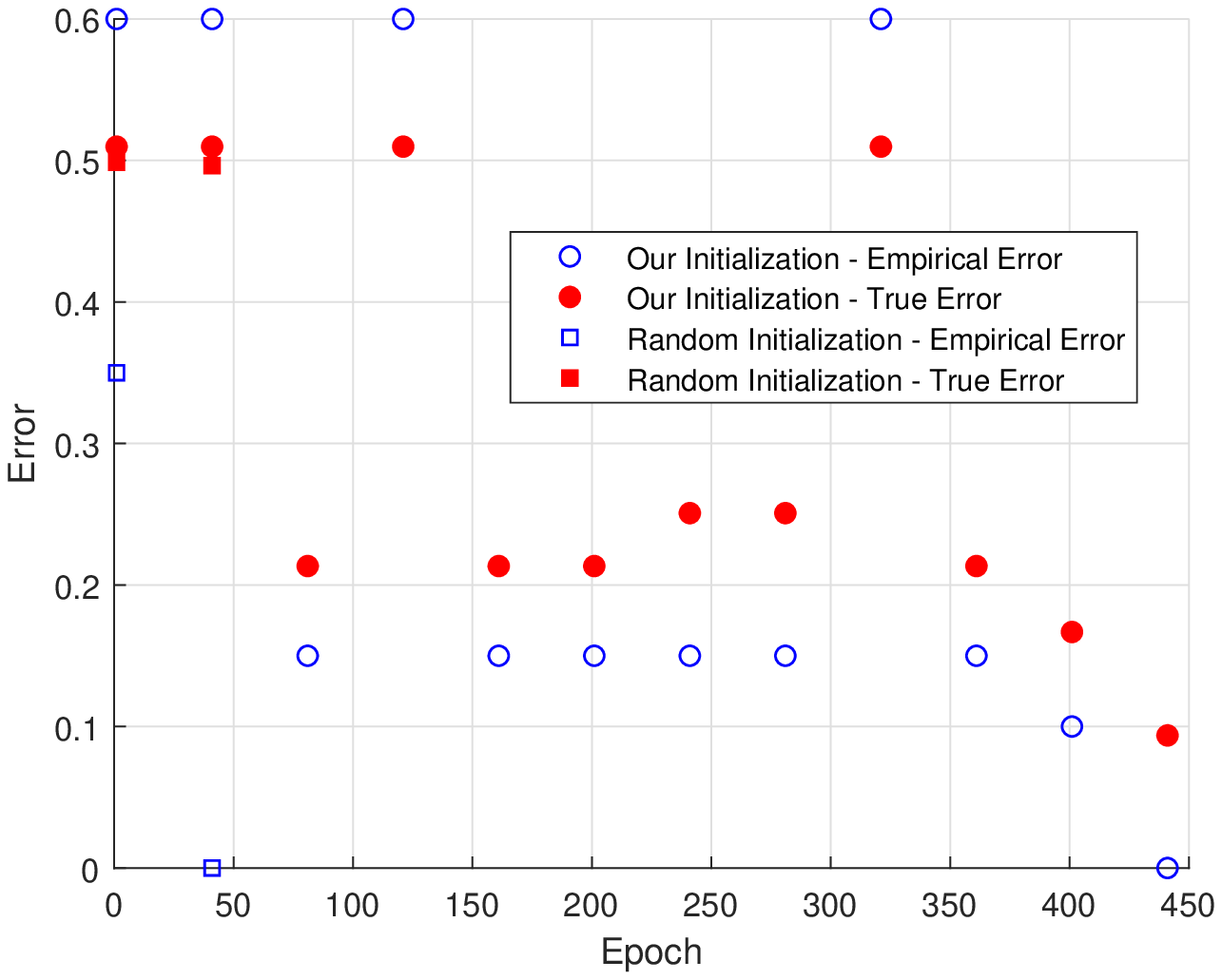} 
\caption{sample of size $n$}
\label{fig:subim1}
\end{subfigure}\hfill
\begin{subfigure}{0.45\textwidth}
\includegraphics[width=0.9\linewidth]{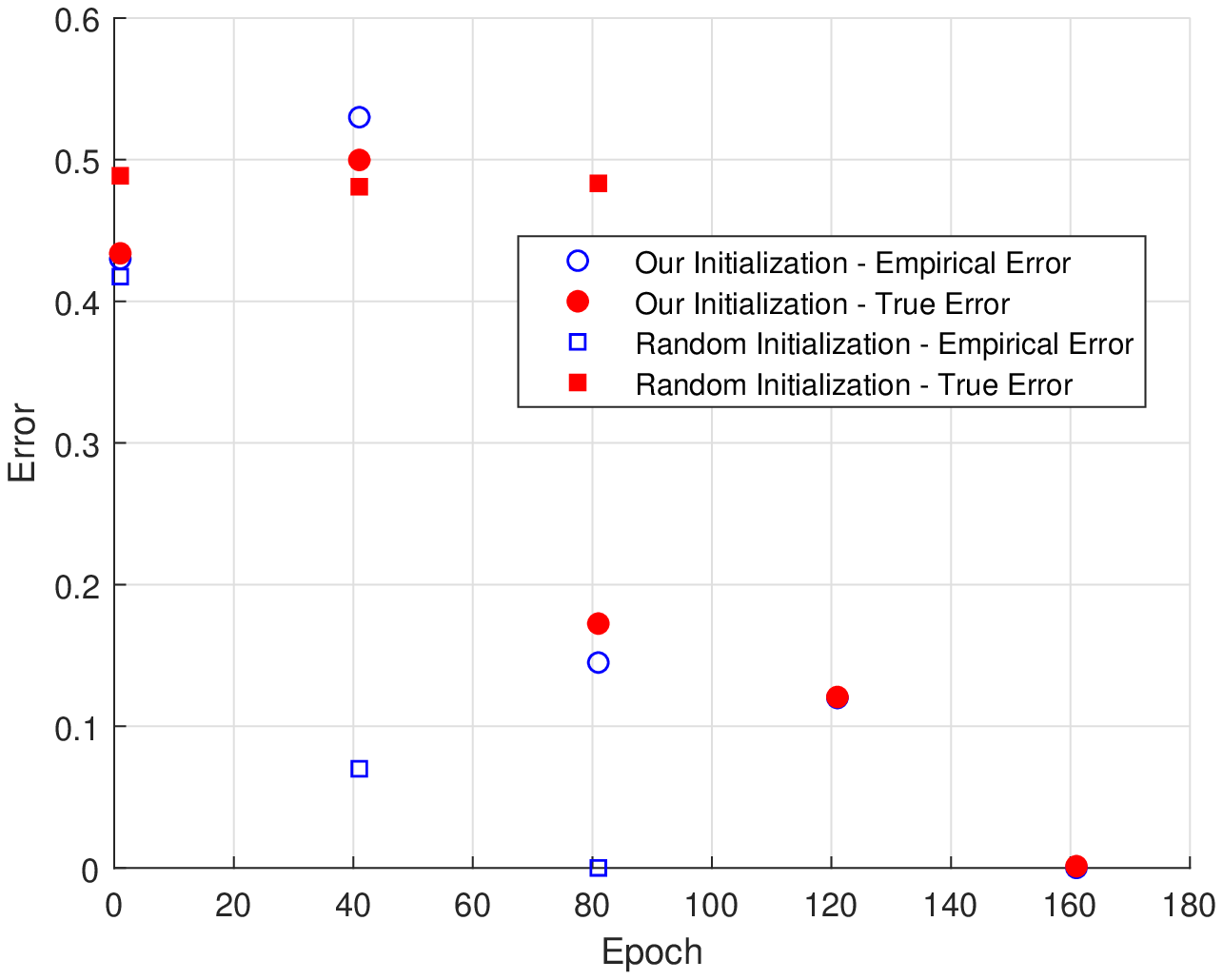}
\caption{sample of size $n^2$}
\label{fig:subim2}
\end{subfigure}

\begin{subfigure}{0.45\textwidth}
\includegraphics[width=0.9\linewidth]{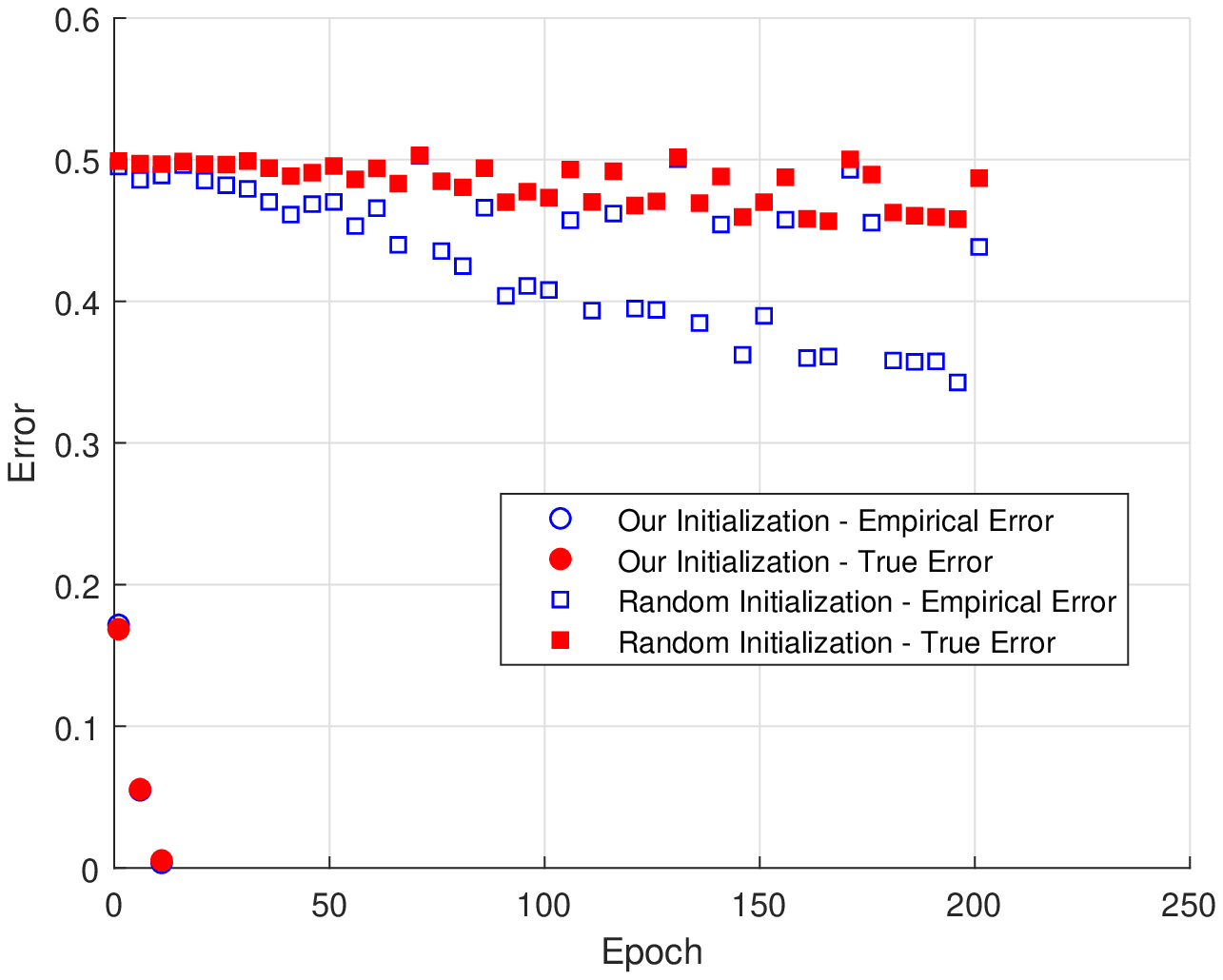} 
\caption{sample of size $n^3$}
\label{fig:subim1}
\end{subfigure}\hfill
\begin{subfigure}{0.45\textwidth}
\includegraphics[width=0.9\linewidth]{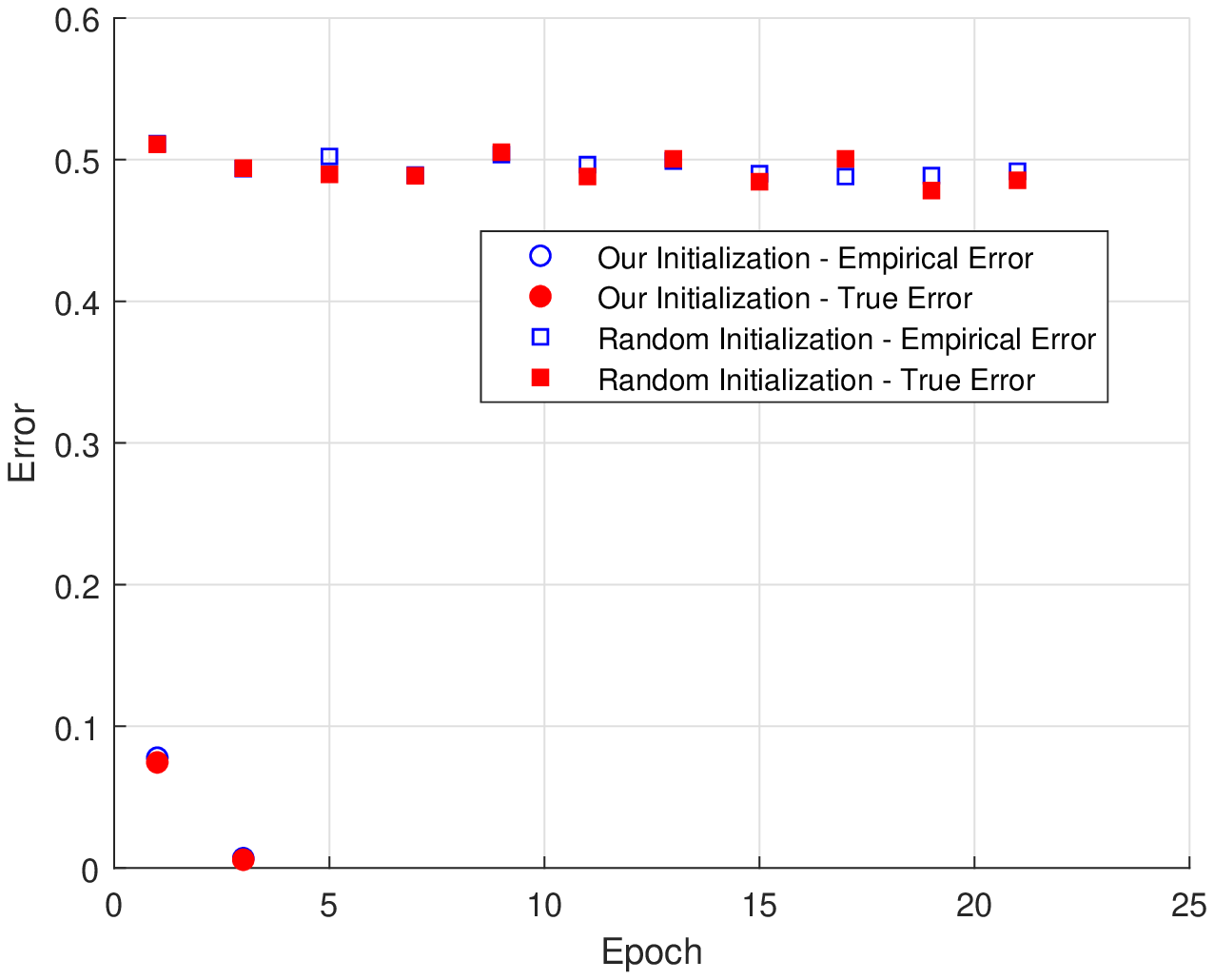}
\caption{sample of size $n^4$}
\label{fig:subim2}
\end{subfigure}
 
\caption{Random symmetric function: error during training for input dimension $n=20$. }
\label{fig_sym}
\end{figure}

\subsection{The Theory in Practice} 
Figure \ref{fig1} demonstrates our theoretical results and also validates  the performance of our initialization. In one setting, we trained only the second layer (freezed the weights of the hidden layer) which essentially  corresponds to the perceptron algorithm. In the second setting, we trained both layers with a step size $h=n^{-6}$ (as the theory suggests). As expected, performance in both cases is similar. We remark that SGD continues to run even after minimizing the empirical error. This happens because of the parameter $\beta>0$.

\subsection{Overstepping the Theory} \label{over}
Here we experiment with two parameters in the proof,
the step size $h$ and the confidence parameter $\beta$.

In Figure \ref{fig2}, we 
used three different step sizes,
two of which much larger than the theory suggests.
We see that the training error converges much faster to zero,
when the step size is larger.
This fast convergence comes at the expense of the true error.
For a large step size, generalization cease
to hold.


Setting $\beta=n^3h$ is a construct in the proof. 
Figure \ref{beta_0} shows that setting $\beta=0$ does not impair the performance.
The difference between theory (requires $\beta >0$)
and practice (allows $\beta=0$) can be explained as follows.
The proof bounds the worst-case
movement of the hidden layer,
whereas in practice an average-case argument suffices.



\subsection{Hard to Learn Parity}  \label{par_hard}  
Figure \ref{fig_par} shows that even for $n=20$, learning parity is hard from a random initialization. When the sample size is small the training error can be nullified but the true error is large. As the sample grows, it becomes much harder for the network to nullify even the training error. With our initialization, both the training error and true error are minimized quickly. Figure \ref{fig_sym} demonstrates the same phenomenon for a random symmetric function.

\subsection{Corruption of Data} 
{

Our initialization also delivers satisfying results when the input data it corrupted. In figure \ref{p_noise}, we randomly perturb (with probability $p = \tfrac{1}{10}$) the labels and use the same SGD
to train the model.
In figure \ref{eps_noise}, we randomly shift every entry of the vectors in the space $X$ by $\epsilon$ that is uniformly distributed in 
$[-0.1,0.1]^n$.

 \begin{figure}[h]
\vskip 0.1in
\begin{center} 
\centerline{\includegraphics[ scale=0.50]{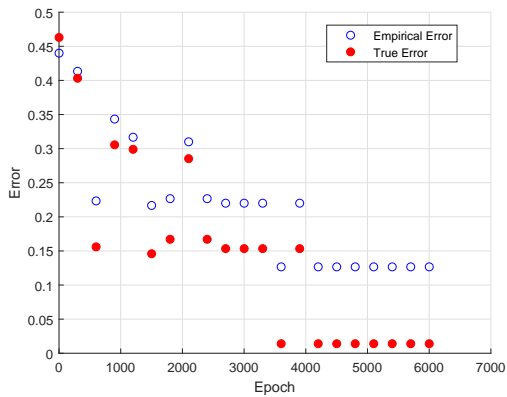}}
\caption{Label error resistance. Labels of the sample were flipped with probability $p=\tfrac{1}{10}$. Sample of size $10n$ whose input dimension is $n=30$.
}
\label{p_noise}
\end{center}
\vskip -0.1in
\end{figure} 
 \begin{figure}[h]
\vskip 0.1in
\begin{center} 
\centerline{\includegraphics[ scale=0.50]{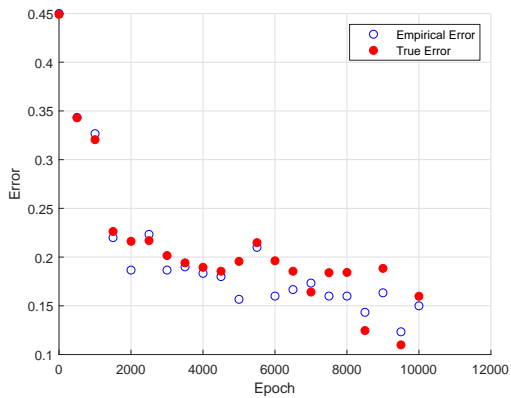}}
\caption{Input Error resistance. All the entries of the vectors in the space were randomly shifted. 
Sample of size $10n$ whose input dimension is $n=30$.
}
\label{eps_noise}
\end{center}
\vskip -0.1in
\end{figure} 

\section{Conclusion}

This work demonstrates that symmetries can play a critical role when designing a neural network. We proved that any symmetric function can be learned by a shallow neural network,
with proper initialization. We demonstrated by simulations 
that this neural network is stable under corruption of data,
and that the small step size is the proof is necessary.

We also demonstrated 
that the parity function or a random symmetric function cannot be learned with random initialization.
How to explain this empirical phenomenon is still an open question. The works \cite{Shamir} and \cite{SongVW} treated parities  
using the language of SQ.
This language obscures the inner mechanism of the network training,
so a more concrete explanation is currently missing.

We proved in a special case
that the standard SGD training of a network 
efficiently produces low true error.
The general problem that remains is
proving similar results for general neural networks. 
A suggestion for future works is to try to identify
favorable geometric states of the network
that guarantee fast convergence and generalization.

\subsection*{Acknowledgements}

We wish to thank Adam Klivans for helpful comments.

\small

\appendix


%
%
%
%
%
%
%
%
%

%

\section{The Modified Perceptron}
 \label{proof_perc}

\begin{proof}[Proof of Theorem~\ref{R-perc}]
Denote by $w^*$ the optimal separating hyperplane with $\lv w^*\rv=1$. It satisfies  $y_i w^* \cdot x_i\geq \gamma $ for all $x_i$. By the definition,
$$w \tu \cdot w^* = w ^{(t-1)}\cdot w^* +y_i w^*\cdot x_i \geq  \gamma ht$$
and
$$\lv w \tu \rv^2=  \lv w ^{(t-1)}\rv^2 +2y_i w ^{(t-1)} x_i h +(\lv x_i \rv h)^2 \leq \lp 2 \beta h+(Rh)^2 \rp t.$$
By Cauchy-Schwarz inequality, $1\geq w \tu \cdot w^*/\lv w\tu  \rv$. So the number of updates is bounded by $$\dfrac{2\beta h +(Rh)^2}{(\gamma h)^2}.$$
At time $t$ the margin of any $x_i$ that does not require an update is at least 
$$ \dfrac{\beta}{ \lv  w\tu  \rv } \geq   \dfrac{\beta}{ \sqrt{  \lp 2 \beta h+(Rh)^2 \rp t } }.  $$
The right hand side is monotonically decreasing function of $t$ so by plugging in the maximal number of updates we see that the minimal margin of the output is at least $$\dfrac{\gamma  \beta h}{2\beta h+(Rh)^2}.$$
\end{proof}

\end{document}